%% file: sampta13-arXiv.tex
\documentclass[a4paper,conference]{IEEEtran}
 \IEEEoverridecommandlockouts
\input{preamble}
\ifCLASSINFOpdf
\else
\fi
\begin{document}
%
\title{Energy-aware adaptive bi-Lipschitz embeddings \thanks{This work was supported in part by the European Commission under Grant MIRG-268398, ERC Future Proof, SNF 200021-132548.}}

\author{
\IEEEauthorblockN{Ali Sadeghian}
\IEEEauthorblockA{LIONS, EPFL, Switzerland}
\and
\IEEEauthorblockN{Bubacarr Bah}
\IEEEauthorblockA{LIONS, EPFL, Switzerland}
\and
\IEEEauthorblockN{Volkan Cevher}
\IEEEauthorblockA{LIONS, EPFL, Switzerland}

}


%


\maketitle

\begin{abstract}
We propose a dimensionality reducing matrix design based on training data with constraints on its Frobenius norm and number of rows. Our design criteria is aimed at preserving the distances between the data points in the dimensionality reduced space as much as possible relative to their distances in original data space. This approach can be considered as a deterministic Bi-Lipschitz embedding of the data points. We introduce a scalable learning algorithm, dubbed AMUSE, and provide a rigorous estimation guarantee by leveraging game theoretic tools. We also provide a generalization characterization of our matrix based on our sample data. We use compressive sensing problems as an example application of our problem, where the Frobenius norm design constraint translates into the sensing energy.

\end{abstract}


%
\IEEEpeerreviewmaketitle

\input{intro}

\input{main_results-arXiv}
\end{document}

%% file: preamble.tex
\usepackage{amsxtra}

\def\d{\delta}

\def\Ph{\boldsymbol{\Phi}}
\def\D{\Delta}
\def\S{\Sigma}

\def\Q{\mathbf{Q}}
\def\f{\mathbf{f}}

\def\N{\mathbf{N}}
\def\V{\mathbf{V}}

\def\RR{\mathbb{R}}
\def\SS{\mathbb{S}}

\newcommand{\points}{\mathcal X}

\newcommand{\loss}{\mathcal L}
\newcommand{\secants}{\mathcal S}

\newcommand{\obs}{\mathbf y}
\newcommand{\x}{\mathbf x}
\newcommand{\X}{\mathbf X}

\newcommand{\linmap}{\boldsymbol{\mathcal{A}}}

\newcommand{\rankr}{r}

\newcommand{\B}{\mathbf B}

\newcommand{\bigO}{\mathcal O}

\newcommand{\bzz}{\mbox{\boldmath{$z$}}}

\newcommand{\mone}{\mbox{\boldmath{$1$}}}

\usepackage{amssymb, amsmath, mathrsfs, latexsym, graphicx}

\newtheorem{corollary}{Corollary}
\newtheorem{proposition}{Proposition}
\newtheorem{definition}{Definition}

\newtheorem{theorem}{Theorem}

\newcommand{\bitem}{\begin{itemize}}
\newcommand{\eitem}{\end{itemize}}

\newcommand{\beqn}{\begin{equation}}
\newcommand{\eeqn}{\end{equation}}
\newcommand{\balign}{\begin{align}}
\newcommand{\ealign}{\end{align}}

%% file: intro.tex
\section{Introduction} \label{sec:intro}

Embedding of high dimensional data into lower dimensions is almost a classical subject. Random projections is one way of doing such embeddings and this method rely on the famous Johnson-Lindenstrauss (JL) lemma \cite{johnson1984extensions}. Recently, JL mappings have also found use in compressed sensing (CS), which is a promising alternative to Nyquist sampling \cite{donoho2006compressed}. The current CS theory uses random, non-adaptive matrices and provide recovery guarantees for highly under sampled signals. An key component in the analysis of CS recovery is the restricted isometry property (RIP), \cite{candes2005decoding,rauhut2010compressive}.
\begin{definition}[\cite{candes2005decoding}]
\label{def:rip}
A matrix $\Ph$ satisfies the RIP of order $k$ if the following holds for all vectors ${\bf z}$, which has at most $k$ nonzero entries (i.e., k-sparse): 
\begin{equation}
\label{eq:rip}
(1-\delta_k)\|{\bf z}\|_2^2 \leq \|\Ph {\bf z}\|_2^2 \leq (1+\delta_k)\|{\bf z}\|_2^2.
\end{equation}
\end{definition}
\noindent The RIP constant (RIC) of $\Ph$ of order $k$ is the smallest $\delta_k$ for which \eqref{eq:rip} holds. In the sequel, we use $\delta$ without explicit reference to $k$ for the RIC.

In this paper, we consider {\em adaptivity} in matrix design. Our setting is as follows: we are given a representative data set which can well-approximate an unknown signal. Using this data set, we would like to design a CS matrix that incorporates time and energy constraints while trying to approximate the best RIP matrix. We provide that the embedding we learn is also generalizable to some extent, that is, if a signal is drawn within $\epsilon$ of the data set, then the matrix will have good RIC. We formulate the matrix learning problem into a semidefinite program (SDP) and propose an algorithm leveraging tools from game theory.

The main contribution of this work is that, to the best of our knowledge, it is the first deterministic design that is adaptive to data, uses RIP and gives provable approximation guarantees. A salient feature of our approach is that the design has the {\em digital fountain} property, which makes it nested, that is, if the measurements are not enough, we can still increase the measurements without changing the previous rows of the matrix. In addition, our approach incorporates an important criteria: the {\em energy constraint}, which may also be important for applications beyond CS. The algorithm we propose is also highly {\em scalable}, that is, it works in linear space in the matrix size because it only keep the matrix factors. Experimentally, using the matrices we design for CS seems promising as our matrices outperform those of random projections.

\textbf{Notation:} We define the set of $k$-sparse vectors as $\Sigma_k :=\{{\bf z} \in \mathbb{R}^{n}: \|{\bf z}\|_0\leq k\}$; and the set $\Xi_r := \left\{\X \in \SS_+^{n\times n} : \text{rank}(\X) \leq r \mbox{ and } \|\X\|_{\mathrm{tr}} \leq \lambda \right\}$ for scalars $r>0$ and $\lambda>0$, where $\SS_+^{n\times n}$ is the set of positive semidefinite (PSD) matrices. We denote the $n$-dimensional simplex by $\D^n$.
\begin{definition}[\cite{hegde2012convex}]
\label{def:secants}
Given $\mathbf{x}_l \in \points \subset \S_k$ we define the set of normalized secants vectors of $\points$ as:
\begin{equation}
\label{eq:sec_data}
\secants(\points) := \left\{ {\bf v}_{ij} = \frac{{\bf x}_i - {\bf x}_j}{\|{\bf x}_i - {\bf x}_j\|_2} \mbox{ for } i\neq j \right\}.
\end{equation}
\end{definition}

%

\textbf{Ouline:}  Section \ref{sec:bkgd} is problem statement with a bit of background; while Section \ref{sec:algo} formulates the problem and presents the algorithm. We analyse the algorithm and give generalization bounds in Section \ref{sec:analysis}, followed by empirical results from simulations and conclusions in Sections \ref{sec:empirics} and \ref{sec:concln} respectively.

%% file: main_results-arXiv.tex
\section{Background and problem description} \label{sec:bkgd}

The CS literature heavily relies on random matrices in establishing recovery guarantees. There has also been also progress in obtaining structured matrices via randomization.  However, for CS to live up to its promise, real applications must be able to use data adaptive matrices. Attempts have been made in this direction that include what is referred to as optimizing projection matrices which entails reducing the correlation between normalised data points (dictionary) of the given data set, see \cite{elad2007optimized,duarte2009learning}. Our work is in this direction as is \cite{hegde2012convex}. Precisely, this work build on what was done in \cite{hegde2012convex} by learning a projection (embedding) matrix from a given data set via the RIP.  However, in sharp contrast to \cite{hegde2012convex}, our solution provides rigorous approximation guarantees. 

To set up the problem, let us assume that we are given a set of $p \gg n$ sample points (training set) $\points = \lbrace {{\bf x}_j} \rbrace_{j = 1}^p$. Then we impose that the embedding matrix we are learning $\Ph$ satisfy RIP on the pairwise distances of the points in $\points$, that is $\Ph$ satisfies \eqref{eq:rip} with {\bf z} replaced by ${\bf x}_i - {\bf x}_j$ for all ${\bf x}_i, {\bf x}_j \in \points$ where $i\neq j$. $\Ph$ is bi-Lipschitz due to the RIP construct. Theoretical guarantees for this approach relies on results from differential geometry, see \cite{hegde2012convex} and the references therein.

If we replaced {\bf z} in \eqref{eq:rip} by ${\bf x}_i - {\bf x}_j$ and normalized the pairwise distances, then the RIP condition \eqref{eq:rip} on $\secants(\points)$ becomes $(1-\d) \leq {\bf v}_{ij}^T \Ph^T\Ph {\bf v}_{ij} \leq (1-\d)$. This expression simplifies to $| {\bf v}_{ij}^T \Ph^T\Ph {\bf v}_{ij}  - 1| \leq \d$ for each $i\neq j$. Re-indexing the ${\bf v}_{ij}$ to ${\bf v}_l$ for $l = 1, \ldots, M$, where $M = \binom{p}{2}$, we form the $M$ secant vectors $\secants(\points) = \lbrace {\bf v}_1, \ldots,{\bf v}_M \rbrace$ into an $n\times M$ matrix $\V = [{\bf v}_1, \ldots,{\bf v}_M]$ and let $\mathbf{B} = \Ph^T \Ph.$ Then we define a linear transform $\linmap: \SS_+^{n\times n} \rightarrow \RR^M$ as:
\begin{equation}
\label{eq:linmap_secant}
\linmap\left( \mathbf{B}\right) := \text{diag}\left(\V^T \mathbf{B} \V\right),
\end{equation}
where $\mathrm{diag}({\bf H})$ denotes a vector of the entries of the principal diagonal of the matrix ${\bf H}$. Note that the rank of $\mathbf{B}$ is the same as that of $\Ph$ and $\mathbf{B}$ is a PSD self-adjoint matrix. In addition, we place a constraint on the energy of $\mathbf{B}$ to be a fixed budget, say $b$, which adds a trace constraint to our problem and in practice may translates for example to having the entries of $\Ph$ to all have a certain magnitude range. So our problem of adaptively learning an energy-aware RIP matrix $\Ph$ and an RIC $\d$ is equivalent to the following trace constrained affine rank minimization (ARM) problem:
\begin{equation}
\label{eq:adaRIP}
\begin{aligned}
& \mathop{\text{min}}_{\mathbf{B}} & & \|\linmap\left( \mathbf{B}\right) - \mathbf{1}_M\|_\infty \\
& \text{s.t.} & &  \mathbf{B} \succeq 0, \quad  \mbox{rank}{(\B)} = r, \quad \mbox{trace}{(\B)} = b.
\end{aligned}
\end{equation}

In \cite{hegde2012convex}, they solve a different problem by constraining the value of $\delta$. Then they use eigen-decomposition to reach a number of samples. We directly take the constraints, design the matrix and give approximation guarantees. In our case, our algorithm returns the factors directly, which reduces the post processing costs such as taking eigendecompositions. 

\section{Proposed design and our algorithm} \label{sec:algo}

Problem \eqref{eq:adaRIP}, as is common practice for ARM problems, can be relaxed as follows:
\begin{eqnarray}
\label{eq:prob1}
\begin{array}{ll}
\displaystyle \min_{\mathbf{B}} & \|\obs - \linmap(\mathbf{B}) \|_\infty \\
\displaystyle \mbox{s.t.} & \mbox{rank}(\mathbf{B}) \leq \rankr \quad \mbox{and} \quad \|\mathbf{B}\|_{\mathrm{tr}} \leq b.
\end{array}
\end{eqnarray}
where $\obs = \mathbf{1}_M$ and $\|\mathbf{B}\|_{\mathrm{tr}} \leq b$ captures the PSD  and the trace constraints.
Based on the work in \cite{jafarpour2011compressive}, we reformulate \eqref{eq:prob1} as a minimax game next. 


\subsection{Reformulation of \eqref{eq:prob1}}\label{sec:arm_inf}

We first define a linear map $\linmap_{+}: \SS^{n\times n} \rightarrow \RR^{2M}$ where $\linmap_{+}(\mathbf{B})$ is a concatenation of $\linmap(\mathbf{B})$ and $-\linmap(\mathbf{B})$ that is:
$\displaystyle \linmap_{+}(\mathbf{B}) = [\linmap(\mathbf{B})^T , -\linmap(\mathbf{B})^T]^T,$
and correspondingly set $\f = [\obs^T, -\obs^T]^T$. Therefore, we have
\begin{multline} 
\label{eqn:min_max}
\|\obs - \linmap(\mathbf{B}) \|_\infty = \max_{i\in [2M]} |\left[\linmap_{+}(\mathbf{B}) - \f\right]_i| = \\
 \max_{i\in [2M]} \mathbf{e}_i^T \left( \linmap_{+}(\mathbf{B}) - \f\right) = \max_{\N \in \D^{2M}} \loss(\N,\mathbf{B}), 
\end{multline}
where $\displaystyle \loss(\N,\mathbf{B}) := \langle \N, \left( \linmap_{+}(\mathbf{B}) - \f\right)\rangle$ and $\mathbf{e}_i$ is the canonical basis vector. The last equality in \eqref{eqn:min_max} is due to the fact that the maximum of a linear program occurs at a boundary point of the simplex $\D^{2M}$. This reduces  problem \eqref{eq:prob1} to a minimax problem:
\begin{equation}
\label{eq:minmax_game_matrix}
\underset{\mathbf{B} \in \Xi_r}{\text{min}} \quad \underset{\N \in \D^{2M}}{\text{max}} \quad \loss(\N,\mathbf{B})
\end{equation}
where $\Xi_r$ is the primal set, $\D^{2M}$ is the dual set and the mapping $\loss: \Xi_r \times \D^{2M} \rightarrow \RR$ is referred to as the {\em loss function} in game theory. We would need the following $\loss_{\max} := \max_{\N,\mathbf{B}} |\loss(\N,\mathbf{B})| = \|\linmap(\mathbf{B})\|_\infty + \|\obs\|_\infty.$ Note that $\linmap_{+}^{*}: \RR^{2M} \rightarrow \SS^{n\times n}$, which is the adjoint of $\linmap_{+}$, can be expressed in terms of the adjoint of $\linmap$, denoted by $\linmap^{*}$, precisely $\linmap_{+}^{*}({\bf w} ) = \linmap^{*}({\bf w}_1 - {\bf w}_2)$ for ${\bf w} = [{\bf w}_1, {\bf w}_2]^T$ where ${\bf w}_1, {\bf w}_2 \in \RR^{M}$.

\subsection{AMUSE algorithm}\label{sec:arm_inf_algo}

We now propose an algorithm that solves the minimax game \eqref{eq:minmax_game_matrix} with provable theoretical guarantees: 
see \textbf{Algorithm 1}. It is important to note that the algorithm works with rank-1 updates $B^t$ (in a matter similar to the conditional gradient descent algorithms \cite{bertsekas1999nonlinear}). As a result, after $r$ iterations, our algorithm returns an estimator $\widehat{\mathbf{B}} = \frac{1}{r}\sum_{t=1}^{r} \mathbf{B}^t$. As we do not explicitly compute the product of the factors, the algorithm is scalable since each factor corresponds to 1 measurement.  Moreover, we bound the recovery error as thus:
\begin{equation*}
\label{error}
\|\linmap(\widehat{\mathbf{B}}) - \obs \|_{\infty} \leq \underset{\mathbf{B} \in \Xi_r}{\min} ~~\|\linmap\left(\mathbf{B}\right) - \obs \|_{\infty} + \mathcal{O}\left(\frac{1}{\sqrt{r}}\right).
\end{equation*}
This is the first approximation bound for obtaining such sensing matrices.

Essentially, the MUSE for ARM (AMUSE) algorithm we propose is a modification of the Multiplicative Update Selector and Estimator (MUSE) algorithm for learning to play repeated games proposed in \cite{jafarpour2011compressive}. The MUSE itself can be thought of as a restatement of the Multiplicative Weights Algorithm (MWA), which in turn uses the Weighted Majority Algorithm, see  \cite{jafarpour2011compressive} and references therein. We also point out also that the multiplicative updating has connections to Frank-Wolfe and related algorithms \cite{clarkson2010coresets}. 

\begin{table}[h]
\centering
\begin{tabular}{l}
\hline
 \textbf{Algorithm 1} MUSE for ARM (AMUSE) \\
 \hline
 \textbf{Input:} $\obs, ~\eta$\\
 \textbf{Output:}  $\widehat{\mathbf{B}} \approx \mathbf{B}^*$ with rank$( \widehat{\mathbf{B}}) \leq r$\\
 \hline
 \textbf{Initialize} $\N^1=\frac{1}{2M}\mathbf{1}_{2M}$ \\
 \textbf{For } $t = 1, \ldots, r$ \textbf{ do}\\
    \quad $1.$ Find $ \mathbf{B}^t = \mathop{\mbox{argmin}}_{\|\mathbf{B}\|_{\mathrm{tr}} \leq 1} ~\loss(\N^t,\mathbf{B})$ \\
    \quad $2.$ Set $~\displaystyle \Q_j^{t+1} = \N_j^t \cdot e^{\frac{\eta \cdot \loss(\mathbf{e}_j,\mathbf{B}^t)}{\loss_{\max}}}$ for $j \in [2M]$  \qquad \qquad\\
    \quad $3.$ Update $ \N^{t+1} = \frac{\Q^{t+1}}{\sum_{j=1}^{2M} \Q_j^{t+1}}$ \\
     \textbf{End for}\\
 \textbf{Return} $ \widehat{\mathbf{B}} = \frac{1}{r}\sum_{t=1}^{r} \mathbf{B}^t$ \\
 \hline
\end{tabular}
\label{tab:armuse_pcode}
\end{table}

Steps 2 and 3 of the loop of AMUSE performs the multiplicative update of the dual variable $\N$ and the update is exactly the same as in MUSE for a given primal variable at iteration $t$, $\mathbf{B}^t$. Therefore the step size $\eta$ remains the same as in the MUSE algorithm, \cite{jafarpour2011compressive}; that is $\eta = \ln \left(1+\sqrt{{2\ln(2M)}/{r}}\right)$. As a result, the theoretical guarantees given in \cite{jafarpour2011compressive} for MUSE also holds for AMUSE. Basically, for a fixed matrix at iteration $t$, $\mathbf{B}^t$, the proof for the multiplicative update in \cite{jafarpour2011compressive} for the vector case remains the same.

Note that the main and crucial difference between AMUSE and MUSE is the first step of the loop where we update our primal variable $\mathbf{B}$ given our dual variable at iteration $t$, $\N^t$, by $\mathbf{B}^t = \mathop{\mbox{argmin}}_{\|\mathbf{B}\|_{\mathrm{tr}} \leq 1} ~\loss(\N^t,\mathbf{B})$. These updates have rank $1$ and hence their  linear combination, $\widehat{\mathbf{B}}$, has rank at most $r$, since rank is sub-additive.

AMUSE is used to approximate problem \eqref{eq:adaRIP} by rescaling to meet the trace constraint. The parameter $\eta$ remain the same and  $\displaystyle \loss_{\max} = 1 + \max_i ~ \max_j ~ v_{ij}^2$ where $v_{ij}$ is the $(i,j)$ entry of $\V$.


\section{Analysis} \label{sec:analysis}


\subsection{AMUSE guarantees} \label{sec:amuse_rip}

The following theorem formalizes our claim that the AMUSE algorithm outputs an approximate solution $\widehat{\mathbf{B}}$ with rank$(\widehat{\mathbf{B}}) \leq r$ with a bounded $\ell_\infty$ loss in the measurement domain after $r$ iterations. The proof of this theorem use Lemma 4.1 of \cite{jafarpour2011compressive}. 
\begin{theorem}
\label{thm:error_bound}
Let AMUSE return $\widehat{\mathbf{B}}$ after $r$ iterations. Then rank$(\widehat{\mathbf{B}}) \leq r$ and $\|\linmap(\widehat{\mathbf{B}}) - \obs \|_\infty$ is at most
\begin{equation*}
\label{eq:error_bound}
\|{\bf e}\|_\infty + \left(1+\sqrt{2}\right) \cdot \left( 2\|\linmap(\widehat{\mathbf{B}})\|_\infty + \|{\bf e}\|_\infty \right) \sqrt{\frac{\ln(2M)}{r}},
\end{equation*}
where ${\bf e}$ measures the perturbation of the linear model. 
\end{theorem}

\begin{proof}
We sketch the proof as follows, for details see \cite{jafarpour2011compressive}. By the definition of $\linmap, ~\obs$ and $\loss$, $\displaystyle \|\linmap(\widehat{\mathbf{B}}) - \obs \|_\infty = \max_{\N} \loss(\N,\widehat{\mathbf{B}})$. Then we first show that $\displaystyle \min_{\mathbf{B}} \max_{\N} \loss(\N,\mathbf{B}) + (1+\sqrt{2}) \loss_{\max} \sqrt{\frac{\ln(2M)}{r}}$ upper bounds $\displaystyle \max_{\N} \loss(\N,\widehat{\mathbf{B}})$, a key ingredient of which is the min-max theorem. Next we deduce that $\displaystyle \min_{\mathbf{B}} \max_{\N} \loss(\N,\mathbf{B}) = \min_{\mathbf{B}} \|\linmap(\mathbf{B}) - \obs \|_\infty \leq \|{\bf e}\|_\infty.$ Then, using the triangle inequality we bound $\loss_{\max}$ by bounding $\|\obs\|_\infty$ as thus: $\loss_{\max} = \|\linmap(\mathbf{B})\|_\infty + \|\obs\|_\infty$ which is upper bounded by $2\|\linmap(\mathbf{B})\|_\infty + \|{\bf e}\|_\infty.$
\end{proof}

Furthermore, we bound the error of the output of AMUSE for the RIP matrix learning problem in Corollary \ref{cor:error_bound} which follows from Theorem \ref{thm:error_bound}.
\begin{corollary}
\label{cor:error_bound}
Let AMUSE learn an RIP matrix $\widehat{\mathbf{B}}$ from a given data set $\points$ after $r$ iterations with RIC $\widehat{\d}$. Assume that the optimal RIP matrix for that $\points$ has RIC $\d^*$. Then $\widehat{\mathbf{B}}$ has rank$(\widehat{\mathbf{B}}) \leq r$ and
\begin{equation*}
\label{eq:error_bound_secant}
\|\linmap(\widehat{\mathbf{B}}) - \mone_M \|_\infty \leq \d^* + 2(1+\sqrt{2}) \sqrt{\frac{\ln(2M)}{r}}.
\end{equation*}
\end{corollary}

This implies that if the optimal solution $\Ph^*$ has RIC $\d^*$ on the training set, then our approximation, $\widehat{\Ph}$, of $\Ph^*$ also satisfies RIP on these data points but with a slightly larger constant $\hat{\d} \leq \d^* + \bigO\left(1/\sqrt{r}\right)$. As the dimensions increase, we approximate the best RIP constant for the given dataset. 


\subsection{Generalization bounds} \label{sec:gen_bounds}


Interestingly, we can provably approximate the optimal RIC even for points that are outside our sample points as stated in the following proposition.
\begin{proposition}
\label{pro:RIC_distant_pts}
Given the pair $\d$ and $\Ph$ as the optimal solution to \eqref{eq:adaRIP}, $\Ph$ applied to any $\bzz$ with $\|\bzz - \x\|_2 \leq \epsilon$ for all $\x \in \points$ and $\epsilon \in [0,1)$ gives an RIC, $\bar{\d}$, bounded as follows:
\begin{equation}
\label{eq:ric_gen_bound}
\bar{\d} \leq (\d + \epsilon)/(1 - \epsilon).
\end{equation}
\end{proposition}

\begin{proof}
Since $\Ph$ is linear w.l.o.g let  $\|\x\|_2 = 1$. For any $\bzz$ such that $\|\bzz - \x\|_2 \leq \epsilon$ and $\|\bzz\|_2 = 1$ then $\|\Ph \bzz\|_2$ can be written as: $$\|\Ph\left(\x - (\bzz - \x)\right)\|_2 \leq \|\Ph \x\|_2 + \|\Ph (\bzz - \x)\|_2$$ using the triangle inequality. Let $\alpha_1$ be the smallest constant such that $\|\Ph \bzz\|_2 \leq (1+\alpha_1)\|\bzz\|_2$ then with the definition of $\d$ from the above inequality we have $$\|\Ph \bzz\|_2 \leq (1+\d)\|\x\|_2 + (1+\alpha_1)\|\bzz - \x\|_2.$$ Evaluating and upper bounding the norms and using the definition of $\alpha_1$ gives $$(1+\alpha_1) \leq (1+\d) + (1+\alpha_1)\epsilon.$$ This simplifies to $\alpha_1 \leq (\d + \epsilon)/(1-\epsilon)$. Similarly, we lower bound $\|\Ph \bzz\|_2$ and have an $\alpha_2$ to be the largest constant such that $\|\Ph \bzz\|_2 \geq (1-\alpha_1)\|\bzz\|_2$, this leads to a bound on $\alpha_2$ as thus: $\alpha_2 \leq (\d + \epsilon)/(1+\epsilon)$. The RIC, $\bar{\d}$, is therefore given by $\max (\alpha_1,\alpha_2)$ and for the values of $\epsilon$ considered this is $\alpha_1$, hence \eqref{eq:ric_gen_bound}. 
\end{proof}


\section{Empirical results} \label{sec:empirics}

We use the synthetic data set of images of translations of white squares in a black background from \cite{hegde2012convex}. In the first experiment we investigate the dependence of RIC we learn on the number of rows (or rank) of the $\Ph$ we learn. Here, we use $M = 2000$ number of secants vectors. We use the same for PCA projected to meet the trace constraint of our problem \eqref{eq:adaRIP} and also generate a random Gaussian matrix also constrained to have trace as our problem. Figure \ref{fig:measurements} displays this comparison, where our method clearly outperforms PCA and random designs.
\begin{figure}[h]
\centering
\includegraphics[width=0.8\columnwidth]{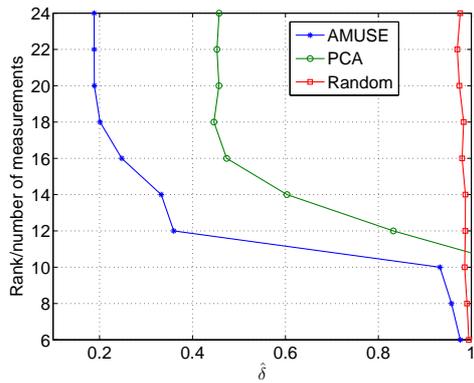}
\caption{A plot of the number of measurements (or rank of the $\widehat{\Ph}$) as a function of the RIC $\widehat{\d}$ for data points with an ambient dimension  $n = 256$.}
\label{fig:measurements}
\end{figure}

In the second experiment we learn a $\Ph$ from the data and use it to encode a randomly selected subset of $\points$ corrupted with Gaussian noise of varying signal-to-noise ratio (SNR). We then do Basis Pursuit denoising to decode these points. For comparison we use a Gaussian matrix with the trace-constrained and compute the mean-square error (MSE) over the subset. The results are displayed in Figure \ref{fig:cs_error}, which show that our approach outperforms the random projections due to its adaptivity to the underlying data manifold. Note that in this experiment, we simply searched over Frobenius norm constraint to approximate the RIC without any energy constraint. 

\begin{figure}[h]
\centering
\includegraphics[width=0.8\columnwidth]{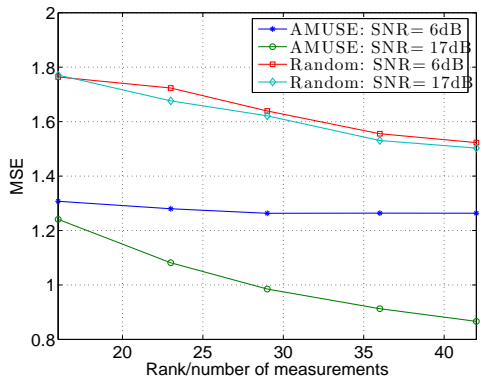}
\caption{CS recovery performance of our adaptive approach compared to energy constrained random projections.}
\label{fig:cs_error}
\end{figure}


\section{Conclusions} \label{sec:concln}

We reformulate the adaptive learning of a data embedding into an optimization problem and propose an algorithm that approximately solves this problem with provable guarantees. We show generalizability of our embedding to a test data set $\epsilon$ away from the training set in terms of the RIC of the embedding matrix learnt. Our experiments show better performance of our derived matrices as compared to random designs with regard to the empirical RIC and CS recovery.